\documentclass[12pt]{article}
\usepackage{fullpage}

\usepackage{hyperref}
\usepackage{url}

\usepackage{amsmath,amsfonts,amssymb}
\usepackage{algorithm,algorithmic}
\usepackage{multirow}
\usepackage{cellspace}
\usepackage{setspace}
\usepackage{bm}
\usepackage{bbm}
\usepackage{subfigure}
\usepackage{graphicx}
\usepackage{tikz,pgfplots}
\usepackage{framed}
\usepackage{wrapfig}

\usepackage{amsthm}
\usepackage{latexsym}

\usepackage[round]{natbib}

\hypersetup{colorlinks,
            linkcolor=blue,
            citecolor=blue,
            urlcolor=magenta,
            linktocpage,
            plainpages=false}

\newtheorem{theorem}{Theorem}[section]
\newtheorem{definition}{Definition}[section]
\newtheorem{lemma}{Lemma}[section]
\newtheorem{corollary}{Corollary}[section]
\newtheorem{assumption}{Assumption}[section]

\newcommand{\startfoo}{%
    \par\medskip
    \begin{mdframed}[linewidth=1pt]%
    \let\figure\figurehere
    \let\endfigure\endfigurehere
    \ignorespaces
}
\newcommand{\stopfoo}{%
    \unskip
    \end{mdframed}%
    \par\medskip
}

\newtheorem{remark}{Remark}

\DeclareMathOperator*{\argmin}{arg\,min}

\def\reals{{\mathcal R}}

\newcommand{\D}{{\mathcal{D}}}

\newcommand{\K}{\mathcal{K}}

\newcommand{\poly}{\mathop{\mbox{\rm poly}}}

\newcommand{\ignore}[1]{}

\def\reals{{\mathbb R}}

\def\bold0{\mathbf{0}}

\newcommand\E{\mbox{\bf E}}

\def\w{\mathbf{w}}
\def\x{\mathbf{x}}
\def\y{\mathbf{y}}
\def\z{\mathbf{z}}

\def\v{\mathbf{v}}

\newcommand{\B}{\mathbb{B}}

\newcommand{\eps}{\varepsilon}

\newcommand{\inner}[1]{\langle #1\rangle}

\newcommand{\ind}[1]{1\!\!1_{#1}}

\renewcommand{\eps}{\epsilon}
\newcommand{\G}{\mathcal{G}}
\renewcommand{\O}{O}
\newcommand{\err}{\widehat{\text{err}}}
\newcommand{\Err}{\mathcal{E}}

\title{Beyond Convexity: Stochastic \\ Quasi-Convex Optimization }

\author{%
Elad Hazan\footnote{Princeton University; \texttt{ehazan@cs.princeton.edu}.}
\and
Kfir Y. Levy\footnote{Technion; \texttt{kfiryl@tx.technion.ac.il}.}
\and 
Shai Shalev-Shwartz\footnote{The Hebrew University; \texttt{shais@cs.huji.ac.il}.}
}

\date{May 2014}                                           

\begin{document}
\maketitle
\begin{abstract} 
 Stochastic convex optimization is a basic and well studied primitive
 in machine learning.  It is well known that convex and Lipschitz
 functions can be minimized efficiently using Stochastic Gradient
 Descent (SGD).  

The Normalized Gradient Descent (NGD) algorithm, is an adaptation of Gradient Descent, which updates according to the direction of the gradients, rather than the gradients themselves.
In this paper we analyze a stochastic version of 
NGD and prove its
convergence to a \emph{global} minimum for a wider class of
functions: we  require the functions to be {\it quasi-convex} and {\it locally-Lipschitz}. 
Quasi-convexity broadens the concept of
unimodality to multidimensions and allows for certain types of
saddle points, which are a known hurdle for first-order optimization
methods such as gradient descent. Locally-Lipschitz functions are only required  to be Lipschitz in a small region
around the optimum. This assumption circumvents gradient explosion, which is another known hurdle for  gradient descent variants. 

 Interestingly, unlike the
vanilla SGD algorithm, the stochastic normalized gradient descent algorithm provably requires a
minimal minibatch size.
\end{abstract} 

\section{Introduction}
\label{section:intro}
The benefits of using the Stochastic Gradient Descent (SGD) scheme for learning could not
be stressed enough. For convex and Lipschitz objectives, SGD is
guaranteed to find an $\eps$-optimal solution within $O(1/\eps^2)$ iterations and requires only an unbiased estimator for the gradient, which is obtained with only one (or  a few) data samples. 
However, when applied to non-convex problems
several drawbacks are revealed. In particular, SGD is widely used
for deep learning \cite{bengio2009learning},  one
of the most interesting fields where stochastic non-convex
optimization problems arise.  Often, the objective in
these kind of problems demonstrates two extreme phenomena: on the one
hand plateaus, i.e., regions with very small gradients; and on the
other hand sharp cliffs, i.e., exceedingly high gradients.  As is
expected, applying SGD to these problems is often reported to yield
unsatisfactory results.

In this paper we analyze a stochastic version of the Normalized
Gradient Descent (NGD) algorithm, which we denote by SNGD. Each iteration of SNGD is
as simple and efficient as SGD, but is much more appropriate for
non-convex optimization problems, overcoming some of the pitfalls
that SGD may encounter.  Particularly, we define a family of
\emph{locally-quasi-convex} and \emph{locally-Lipschitz} functions, and prove
that SNGD is suitable for optimizing such objectives.

Local-Quasi-convexity
is a generalization of unimodal functions to multidimensions, which
includes quasi-convex, and  convex functions as a subclass. 
Locally-Quasi-convex  functions allow for certain types of plateaus and saddle points which are difficult for SGD and other gradient descent variants. 
Local-Lipschitzness is a generalization of
Lipschitz functions that only assumes Lipschitzness in a small region
around the minima, whereas farther away the gradients may be unbounded. Gradient explosion is, thus, another difficulty that is successfully tackled by SNGD and poses difficulties for other stochastic gradient descent variants. 

Our contributions:

\begin{itemize}
\item We introduce \emph{local-quasi-convexity}, a property that extends 
quasi-convexity and captures unimodal functions which are not quasi-convex. 
We prove that NGD finds an $\eps$-optimal minimum for such functions within  $\O(1/\eps^2)$ iterations.
As a special case, we show that the above rate can be attained for quasi-convex functions that are Lipschitz in an
  $\Omega(\eps)$-region around the optimum (gradients may be
  \emph{unbounded} outside this region). For objectives that are also smooth in an $\Omega(\sqrt{\eps})$-region around the optimum,
  we  prove a faster rate of $\O(1/\eps)$.

\item We introduce a new setup: stochastic optimization of locally-quasi-convex functions;
and show that this setup captures Generalized Linear Models (GLM) regression, \cite{GLM}.
For this setup, we devise a stochastic version of NGD (SNGD), and show that it converges within 
 $\O(1/\eps^2)$ iterations to an $\eps$-optimal minimum. 

%

\item The above positive result requires that at each iteration of
  SNGD, the gradient should be estimated using a minibatch of a
  minimal size. We provide a negative result showing that if the
  minibatch size is too small then the algorithm might indeed diverge.

%

\item We report experimental results supporting our theoretical
  guarantees and demonstrate an accelerated convergence attained by SNGD. 

\end{itemize}

\subsection{Related Work}

Quasi-convex optimization problems arise in numerous fields, spanning
economics \cite{varian1985price, laffont2009theory}, industrial
organization \cite{wolfstetter1999topics} , and computer vision
\cite{ke2007quasiconvex}.
It is well known that quasi-convex optimization tasks can be solved
by a series of convex feasibility problems \cite{boyd2004convex};
However, generally solving such feasibility problems may be very
costly \cite{goffin1996complexity}.

There exists a rich literature concerning quasi-convex optimization in the offline case, 
\cite{polyak1967general,zabotin1972minimization,khabibullin1977method,sikorski1986quasi}.
A pioneering paper by \cite{nesterov1984minimization}, was the first
to suggest an efficient algorithm, namely Normalized Gradient Descent,
and prove that this algorithm attains $\eps$-optimal solution within
$O(1/\eps^2)$ iterations given a differentiable quasi-convex objective.
This work was later extended by \cite{kiwiel2001convergence}, showing
that the same result may be achieved assuming upper semi-continuous
quasi-convex objectives. In \cite{konnov2003convergence} it was shown
how to attain faster rates for quasi-convex optimization, but they
assume to know the optimal value of the objective, an assumption that generally
does not hold in practice.

Among the deep learning community there have been several attempts to tackle 
gradient-explosion/plateaus.
Ideas spanning gradient-clipping \cite{pascanu2013difficulty}, smart initialization \cite{doya1993bifurcations}, and more, \cite{martens2011learning}, have shown to improve training in practice.
Yet, non of these works provides a theoretical analysis showing better convergence guarantees.

To the best of our knowledge, there are no previous results on
stochastic versions of NGD, neither results regarding locally-quasi-convex/locally-Lipschitz functions.


%
Gradient descent with fixed step sizes, including its stochastic variants, is known to perform poorly when the gradients are too small in a plateau area of the function, or alternatively when the other extreme happens: gradient explosions. These two phenomena have been reported in certain types of non-convex optimization, such as training of deep networks. 

Figure \ref{figure:cliffs} depicts a one-dimensional family of functions for which GD behaves provably poorly. With a large step-size, GD will hit the cliffs and then oscillate between the two boundaries. Alternatively, with a small step size, the low gradients will cause GD to miss the middle valley which has constant size of ${1}/{2}$. 
On the other hand, this exact function is quasi-convex and locally-Lipschitz, and hence the NGD algorithm provably converges to the optimum quickly.

%
\begin{figure}[t] \label{comp1}
	\begin{center}
		\includegraphics[width=4in]{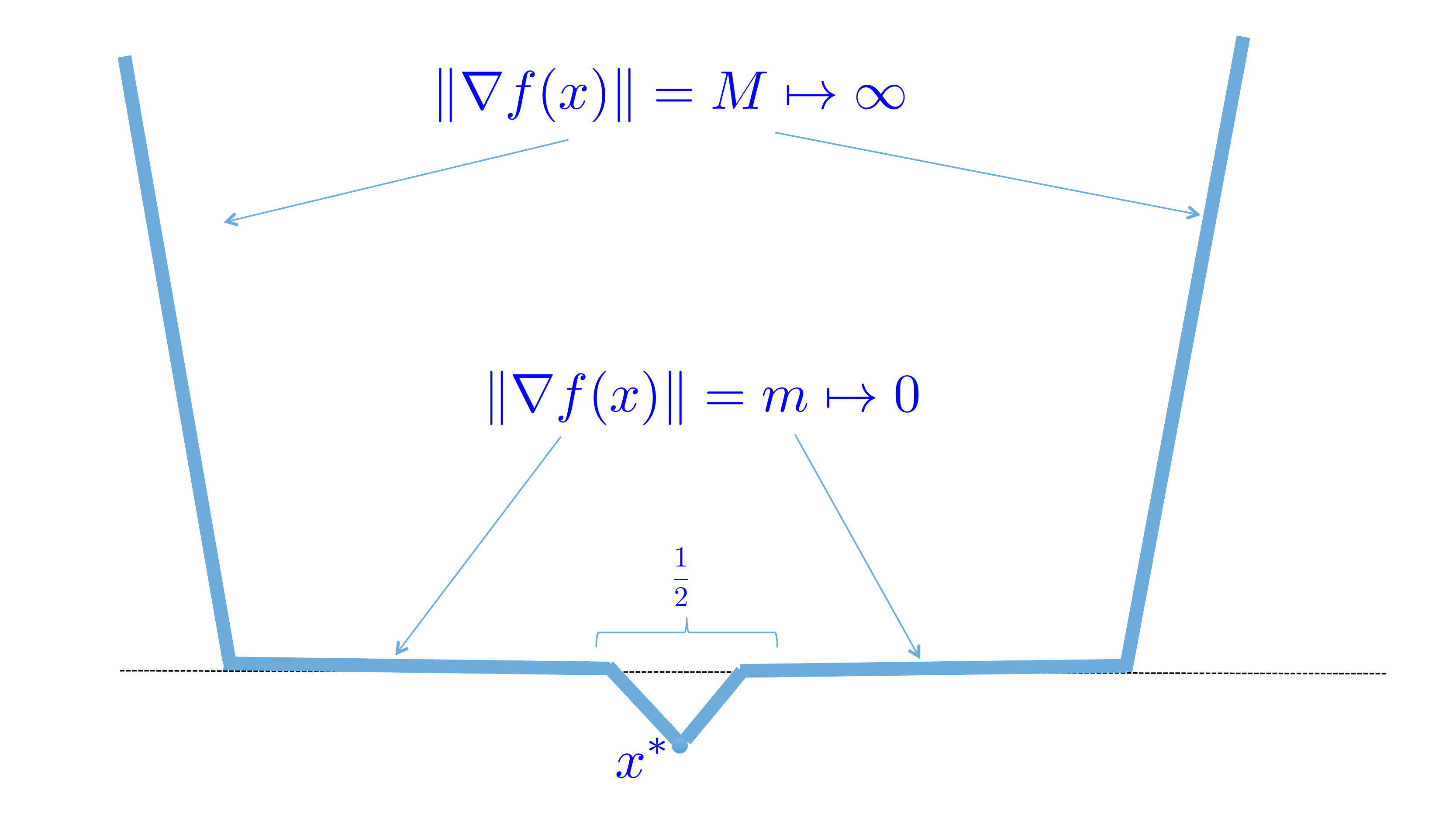}
	\end{center}
	\caption{A quasi-convex Locally-Lipschitz function with plateaus and cliffs.}  \label{figure:cliffs}
\end{figure}

\section{Definitions and Notations}
We use $\|\cdot\|$ to denote the Euclidean norm.
$\B_d(\x,r)$  denotes the $d$ dimensional Euclidean ball of radius $r$, centered around
$\x$, and $\B_d: =\B_d(0,1)$. $[N]$  denotes the set $\{1,\ldots,N \}$.

For simplicity, throughout the paper we always assume that functions are
\emph{differentiable} (but if not stated explicitly, we do not assume any
bound on the norm of the gradients). 
%

\begin{definition}\textbf{(Local-Lipschitzness and Local-Smoothness)}
\label{def:localSmoothness}
  Let $\z \in\reals^d$, $G,\eps\geq 0$. A function $f: \K \mapsto \reals$ is called $(G,\eps,\z)$-Locally-Lipschitz if 
  for every   $\x,\y \in \B_d(\z,\eps)$,  we have
   $$|f(\x)-f(\y)| \le G\|\x-\y\|~.$$
Similarly, the function is $(\beta,\eps,\z)$-locally-smooth if
  for every  $\ \x,\y \in \B_d(\z,\eps)$ we have,
   $$|f(\y)-f(\x)-\inner{\nabla f(\y),\x-\y} |\leq \frac{\beta}{2} \|\x-\y\|^2~.$$
\end{definition}

Next we define quasi-convex functions:
\begin{definition}\textbf{(Quasi-Convexity)} \label{def:qc1}
We say that a function $f: \reals^d \mapsto \reals$ is quasi-convex  if  $\forall \x,\y\in\reals^d$, such that $f(\y)\leq f(\x)$, it follows that
$$\inner{\nabla f(\x), \y-\x}\leq 0~.$$ 
We further say that $f$ is strictly-quasi-convex, if it is quasi-convex and its gradients vanish only at the global minima, i.e., $\forall \y:~ f(\y)>\min_{\x\in \reals^d}f(\x)  ~\Rightarrow~ \|\nabla f(\y)\|>0 $.
\end{definition}
Informally, the above characterization  states that the (opposite) gradient of a quasi-convex function directs us in a \emph{global} descent direction.
Following is an equivalent  (more common) definition: 
\begin{definition}\textbf{(Quasi-Convexity)}\label{def:qc2}
We say that a function $f: \reals^d \mapsto \reals$ is quasi-convex  if any $\alpha$-sublevel-set of $f$ is convex, i.e., $\forall \alpha\in \reals$
the set 
$$\mathcal{L}_{\alpha}(f) =\{ \x: f(\x)\leq \alpha\}\quad \text{is convex.}$$ 
\end{definition}
The equivalence between the above definitions can be found in \cite{boyd2004convex}, for completeness we provide a proof in Appendix~\ref{apdx:defEquiv}.
During this paper we  denote the sublevel-set of $f$ at $\x$ by 
\begin{align}
\label{equation:sublevelset}
S_f(\x) = \{\y : f(\y) \le f(\x)\}~. 
\end{align}

\section{Local-Quasi-Convexity} \label{sec:LocalQC}
Quasi-convexity does not fully capture the notion of unimodality in several dimension. As an example let $\x = (x_1,x_2)\in[-10,10]^{2}$,  and consider the function
\begin{equation}\label{eq:FuncExample}
g(\x)=(1+e^{-x_1})^{-1}+(1+e^{-x_2})^{-1}~.
\end{equation}
It is natural to consider $g$ as   unimodal since it acquires no local minima but for the unique global minima at $\x^* = (-10,-10)$. 
However, $g$ is not quasi-convex:  consider the points $\x = (\log16, -\log4), \y= ( -\log4,\log16)$, which belong to the $1.2$-sub-level set, 
 their average does not belong to the same sub-level-set since $g(\x/2+\y/2) = 4/3$.

Quasi-convex functions always enable us to \emph{explore}, meaning that the gradient 
always directs us in a global descent direction. Intuitively, from an optimization point of view, we only need such a direction whenever 
we do not \emph{exploit}, i.e., whenever we are not approximately optimal.

In what follows we define local-quasi-convexity, a property that enables us to either \emph{explore/exploit}. This property
\footnote{Definition~\ref{def:LocalQC} can be generalized in a manner that captures a broader range of scenarions (e.g. the Perceptron problem), we defer this definition to Appendix~\ref{app:def:LocalQC}.}
 captures a wider class of unimodal function (such as $g$ above) rather than mere quasi-convexity.  Later we justify this definition by showing that it captures Generalized Linear Models (GLM) regression, see \cite{GLM,Kalai}.

\begin{definition}\textbf{(Local-Quasi-Convexity)}\label{def:LocalQC}
Let $\x,\z\in \reals^d$,  $\kappa,\eps>0$. 
 We say that  $f: \reals^d \mapsto \reals$ is  $(\eps,\kappa,\z)$-Strictly-Locally-Quasi-Convex (SLQC) in $\x$, if at least one of the following applies:
\begin{enumerate}
\item $f(\x)- f(\z)\leq  \eps~.$
\item $\|\nabla f(\x)\|> 0$, and for every $\y\in\B(\z,\eps/\kappa)$ it holds that $\inner{\nabla f(\x), \y-\x}\leq 0~.$
\end{enumerate}
\end{definition}
Note  that if $f$ is  $G$-Lispschitz and strictly-quasi-convex function, then $\forall \x,\z\in \reals^d,\; \forall \eps> 0$, it holds that $f$ is $(\eps,G,\z)$-SLQC in $\x$.
Recalling the function $g$ that appears in Equation~\eqref{eq:FuncExample}, then it can be shown that $\forall \eps\in(0,1], \forall \x  \in[-10,10]^2$ then this function is $(\eps,1,\x^*)$-SLQC in  $\x$, where $\x^* = (-10,-10)$ (see Appendix~\ref{app:localQuasi_g}).
\subsection{Generalized Linear Models (GLM)}
\subsubsection{The Idealized GLM}\label{sec:IdealizedPercept}
In this setup we have a collection of $m$ samples $\{(\x_i,y_i) \}_{i=1}^m \in \B_d\times [0,1]$, and an activation function $\phi:\reals \mapsto\reals$.  We are guaranteed to
have $\w^* \in \reals^d$ such that: $ y_i = \phi\inner{\w^*,\x_i},\; \forall i\in[m]$ (we denote $\phi\inner{\w,\x}:=\phi(\inner{\w,\x})$).
The performance of a predictor $\w\in\reals^d$, is measured  by the average square error over all samples.
\begin{align}\label{eq:ErrPerceptron}
\err_m(\w) = \frac{1}{m}\sum_{i=1}^m \left(y_i-\phi\inner{\w,\x_i}\right)^2~. 
\end{align}
 In \cite{Kalai}  it is shown that the Perceptron problem with $\gamma$-margin is a  private case of GLM regression.

The sigmoid function $\phi(z) = (1+e^{-z})^{-1}$ is a popular activation function in the field of deep learning.  
The next lemma states that in the  idealized GLM problem with sigmoid activation,  then the error function is SLQC (but not quasi-convex). As we will see in Section \ref{sec:Setup_NGD} this implies that Algorithm~\ref{algorithm:NGD} finds an $\eps$-optimal minima of $\err_m(\w)$ within $\poly(1/\eps)$ iterations. 
\begin{lemma}\label{lem:IdealizedPercept}
Consider the idealized GLM problem with the sigmoid activation,  and assume that  $\|\w^*\|\leq W$. Then the error function appearing in Equation~\eqref{eq:ErrPerceptron} is $(\eps,e^W,\w^* )$-SLQC in $\w$, $\forall \eps>0,\;\forall \w \in\B_d(0,W)$ (But  it is not generally quasi-convex).
\end{lemma}
We defer the proof to Appendix~\ref{app:Proof_lem:IdealizedPercept}

\subsubsection{The Noisy GLM}\label{sec:IdealizedPerceptNoisy}
In the noisy GLM setup (see~\cite{GLM,Kalai}), we may draw i.i.d. samples $\{(\x_i,y_i)\}_{i=1}^m \in \B_d\times [0,1]$, from an unknown distribution $\D$. We assume that  there exists a predictor $\w^*\in\reals^d$ such that
$\E_{(\x,y)\sim\D}[y\vert \x] = \phi\inner{\w^*,\x}$, where $\phi$ is an activation function. Given $\w\in \reals^d$ we define its expected error as follows:
$$\Err(\w) = \E_{(\x,y)\sim\D}(y -\phi\inner{\w,\x})^2 ~,$$
and it can be shown that  $\w^*$ is a global minima of $\Err$.
We are interested in schemes that  obtain an $\eps$-optimal minima to $\Err$, within $\text{poly}(1/\eps)$ samples and optimization steps.
Given $m$ samples from $\D$, their empirical error $\err_m(\w)$,  is defined as in Equation~\eqref{eq:ErrPerceptron}.

The following lemma states that in this setup, letting $m= \Omega(1/\eps^2)$, then $\err_m$ is SLQC with high probability. This property will enable us to apply Algorithm~\ref{algorithm:SNGD}, to obtain an $\eps$-optimal minima to $\Err$,  within $\text{poly}(1/\eps)$ samples from $\D$, and $\text{poly}(1/\eps)$ optimization steps. 

\begin{lemma}\label{lem:IdealizedPerceptNoisy}
Let $\delta,\eps\in (0, 1)$. Consider the noisy GLM problem with the sigmoid activation, and assume that  $\|\w^*\|\leq W$. 
Given a  fixed point $\w\in \B(0,W)$, then w.p.$\geq 1-\delta$, after $m\geq \frac{8e^{2W}(W+1)^2}{\eps^2} \log(1/\delta)$ samples,  the empirical error function appearing in Equation~\eqref{eq:ErrPerceptron} is $(\eps,e^W,w^* )$-SLQC in $\w$.
\end{lemma}
Note that if we had  required the SLQC to hold $\forall \w\in \B(0,W)$, then we would need the number of samples to depend on the dimension, $d$, which we would like to avoid. 
Instead, we require SLQC to hold for a fixed $\w$. This satisfies the conditions of Algorithm~\ref{algorithm:SNGD}, enabling us to find an $\eps$-optimal solution with a sample complexity that is independent of the dimension. 
We defer the proof of Lemma~\ref{lem:IdealizedPerceptNoisy}  to Appendix~\ref{app:Proof_lem:IdealizedPerceptNoisy}

\section{NGD for Locally-Quasi-Convex Optimization}\label{sec:Setup_NGD}
\begin{algorithm}[h]
\caption{Normalized Gradient Descent (NGD) }
\label{algorithm:NGD}
\begin{algorithmic}
\STATE \textbf{Input}: \#Iterations $T$, $\x_1\in \reals^d$, learning rate $\eta$
\FOR{$t=1 \ldots T$ }
\STATE {Update:}
$$\x_{t+1}= \x_{t}-\eta \hat{g}_{t}\; \text{ where }  g_t= \nabla f(\x_t),\; \hat{g}_t = \frac{g_t}{\|g_t\|}$$
\ENDFOR
\STATE \textbf{Return:}
$\bar{\x}_T = \argmin_{\{\x_1,\ldots,\x_T\} }f(\x_t)$
\end{algorithmic}
\end{algorithm}
Here we present the NGD algorithm, and prove the convergence rate of this algorithm for SLQC  objectives. 
Our analysis is simple, enabling us  to extend the convergence rate presented in \cite{nesterov1984minimization} beyond quasi-convex functions.
We then show that quasi-convex and \emph{locally-Lipschitz} objective are SLQC, implying that NGD converges even if the gradients are unbounded outside a small region around the minima.   
For quasi-convex  and \emph{locally-smooth} objectives, we show that NGD attains a faster convergence rate.

NGD is presented in Algorithm~\ref{algorithm:NGD}.  NGD
is similar to GD, except we normalize the gradients. It is
intuitively clear that to obtain robustness to plateaus (where the
gradient can be arbitrarily small) and to exploding gradients (where
the gradient can be arbitrarily large), one must ignore the size of
the gradient. It is more surprising that the information in the
direction of the gradient suffices to guarantee convergence.

Following is the main theorem of this section:
\begin{theorem}
\label{theorem:NGD}
Fix $\eps>0$,
 let $f:\reals^d\mapsto \reals$, and $\x^*\in \arg\min_{\x\in\reals^d}f(\x)$.
Given that $f$ is $(\eps,\kappa,\x^*)$-SLQC in every $\x \in \reals^d$.
 Then running the NGD algorithm with $T \geq {\kappa^2\|\x_1 - \x^*\|^2}/{\eps^2}$, and $\eta=\eps/\kappa$, we have that:
 $f(\bar{\x}_T)-f(\x^*) \le   \eps$.
\end{theorem}
Theorem~\ref{theorem:NGD} states that $(\cdot,\cdot,\x^*)$-SLQC functions admit $\text{poly}(1/\eps)$  convergence rate using NGD. The intuition
behind this lies in Definition~\ref{def:LocalQC}, which asserts that at a point $\x$ either the (opposite) gradient points out
a global optimization direction, or we are already $\eps$-optimal. Note that the requirement of $(\eps,\cdot,\cdot)$-SLQC in any $\x$ is not restrictive, 
as we have seen in Section~\ref{sec:LocalQC}, there are interesting examples of  functions that admit this property  $\forall \eps \in [0,1]$, and for any $\x$.

For simplicity we have presented NGD for unconstrained problems. Using projections we can 
easily extend the algorithm  and its analysis for constrained optimization over convex sets. 
This will enable to achieve convergence of  $O(1/\eps^2)$ for the objective presented in Equation~\eqref{eq:FuncExample}, and the idealized GLM problem 
presented in Section~\ref{sec:IdealizedPercept}.

We are now ready to prove Theorem~\ref{theorem:NGD}:
\begin{proof}[Proof of Theorem~\ref{theorem:NGD}]
First note that if the gradient of $f$ vanishes at $\x_t$, then by
the SLQC assumption we must have that $f(\x_t)-f(\x^*)\le \eps$.
 Assume next that we perform $T$ iterations and the gradient of $f$ at $\x_t$
never vanishes in these iterations. 
Consider the update rule of NGD (Algorithm~\ref{algorithm:NGD}), then by standard algebra we get,
\[
\|\x_{t+1}-\x^*\|^2 = \|\x_t - \x^*\|^2  - 2 \eta
\inner{\hat{g}_t,\x_t-\x^*} + \eta^2~.
\]
Assume that  $\forall t\in[T]$ we have $f(\x_t)-f(\x^*) >
\eps$. Take
$\y = \x^* + \left(\eps/\kappa\right) \hat{g}_t$, 
 and observe that $\|\y-\x^*\| \le \eps/\kappa$. 
 The $(\eps,\kappa,\x^*)$-SLQC assumption implies that $\inner{\hat{g}_t,\y-\x_t} \le 0 $, and therefore 
\[
\inner{\hat{g}_t,\x^* + \left(\eps/\kappa \right) \hat{g}_t-\x_t} \le 0 
~\Rightarrow~ \inner{\hat{g}_t,\x_t - \x^*} \ge \eps/\kappa~.
\]
Setting  $\eta = \eps/\kappa$, the above implies,
\begin{align*}
\|\x_{t+1}-\x^*\|^2 &\le \|\x_t - \x^*\|^2  - 2 \eta
\eps/\kappa + \eta^2 \\
&= \|\x_t - \x^*\|^2  - \eps^2/\kappa^2~.
\end{align*}
Thus, after $T$ iterations for which $f(\x_t)-f(\x^*) > \eps$ we get 
\[
0 \le \|\x_{T+1}-\x^*\|^2 \le \|\x_1 - \x^*\|^2 - T  \eps^2/\kappa^2 ~,
\]
Therefore, we must have $T \le {\kappa^2 \|\x_1 - \x^*\|^2}/{\eps^2} ~.$
\end{proof}
\subsection{Locally-Lipschitz/Smooth Quasi-Convex Optimization}
It can be shown that strict-quasi-convexity and $(G,\eps/G,\x^*)$-local-Lipschitzness of $f$ implies that $f$ is  $(\eps,G,\x^*)$-SLQC  $\forall \x\in\reals^d$, $\forall \eps\geq 0$,
and $\x^* \in \arg\min_{\x\in\reals^d}f(x)$ (see Appendix~\ref{app:LocallyLip2LocallyQuas}).
Therefore the following is a direct corollary of Theorem~\ref{theorem:NGD}:
\begin{corollary}
\label{cor:NGD}
Fix $\eps>0$,
 let $f:\reals^d\mapsto \reals$, and $\x^*\in \arg\min_{\x\in\reals^d}f(\x)$.
Given that $f$ is strictly quasi-convex and $(G,\eps/G,\x^*)$-locally-Lipschitz.
 Then running the NGD algorithm with $T \geq {G^2\|\x_1 - \x^*\|^2}/{\eps^2}$, and $\eta=\eps/G$, we have that:
 $f(\bar{\x}_T)-f(\x^*) \le   \eps$.
\end{corollary}
In case $f$ is also locally-smooth, we state an even faster rate:
\begin{theorem}
\label{theorem:NGDSmooth}
Fix $\eps>0$,
 let $f:\reals^d\mapsto \reals$, and $\x^*\in \arg\min_{\x\in\reals^d}f(\x)$.
Given that $f$ is strictly quasi-convex and $(\beta,\sqrt{2\eps/\beta},\x^*)$-locally-smooth.
 Then running the NGD algorithm with $T\geq {\beta \|\x_1-\x^*\|^2}/{2\eps}$, and $\eta = \sqrt{2\eps/\beta}$, we have that:
 $f(\bar{\x}_T)-f(\x^*) \le   \eps$.
\end{theorem}
We prove Theorem~\ref{theorem:NGDSmooth} in Appendix~\ref{app:proofNGDsmooth}.
\begin{remark}
  The above corollary (resp. theorem) implies that $f$ could have \emph{arbitrarily large
    gradients and second derivatives} outside $\B(\x^*,\eps/G)$ (resp. $\B(\x^*,\sqrt{2\eps/\beta})$), yet NGD is still ensured to
  output an $\eps$-optimal point within ${G^2 \|\x_1-\x^*\|^2}/{\eps^2}$ (resp. $\beta \|\x_1-\x^*\|^2/{2\eps}$)
  iterations. We are not familiar with a similar guarantee for GD
  even in the convex case.
\end{remark}

\section{SNGD for Stochastic SLQC Optimization}\label{sec:Setup_SNGD}
Here we describe the setting of stochastic SLQC optimization. Then we describe our SNGD algorithm which is ensured to yield an $\eps$-optimal solution 
 within $\poly(1/\eps)$ queries. We also show that the (noisy) GLM problem, described in Section~\ref{sec:IdealizedPerceptNoisy} is an instance of stochastic SLQC optimization,
 allowing us to provably solve this problem within $\poly(1/\eps)$ samples and optimization steps  using SNGD.
 \paragraph{The stochastic SLQC optimization Setup:}
Consider the problem of minimizing a function $f:\reals^d\mapsto\reals$, and assume there exists a distribution over functions $\D$, such that: 
$$f(\x) := \E_{\psi\sim\D}[\psi(\x)] ~.$$
We assume that we may access $f$ by  randomly sampling minibatches of size $b$, and querying the gradients of these minibatches.
Thus, upon querying a point $\x_t \in \reals^d$, a random minibatch $\{\psi_i\}_{i=1}^b\sim\D^b$ is sampled, and we receive $\nabla f_t(\x_t)$, where $f_t(\x) = \frac{1}{b}\sum_{i=1}^b \psi_i(\x)$.
\begin{algorithm}[t]
\caption{Stochastic Normalized Gradient Descent (SNGD) }
\label{algorithm:SNGD}
\begin{algorithmic}
\STATE \textbf{Input}: \#Iterations $T$, $\x_1\in \reals^d$, learning rate $\eta$, minibatch size $b$
\FOR{$t=1 \ldots T$ }
\STATE {Sample:}  $\{\psi_i\}_{i=1}^b\sim\D^b$, and define,
$$f_t(\x) = \frac{1}{b}\sum_{i=1}^b \psi_i(\x)$$ 
\STATE {Update:}
$$\x_{t+1}= \x_{t}-\eta \hat{g}_{t}\; \text{ where }  g_t= \nabla f_t(\x_t),\; \hat{g}_t = \frac{g_t}{\|g_t\|}$$
\ENDFOR
\STATE \textbf{Return:}
$\bar{\x}_T = \argmin_{\{\x_1,\ldots,\x_T\} }f_t(\x_t)$
\end{algorithmic}
\end{algorithm}
We make the following assumption regarding the minibatch averages:
\begin{assumption}\label{ass:SLQC}
Let $T,\eps,\delta>0$, $\x^*\in \arg\min_{\x\in\reals^d}f(\x)$. There exists $\kappa>0$, and a function $b_0: \reals^3 \mapsto\reals$, that for $b\geq b_0(\eps,\delta,T)$ then w.p.$\geq1-\delta$ and $\forall t\in[T]$, 
the minibatch average  $f_t(\x)=\frac{1}{b}\sum_{i=1}^b \psi_i(\x)$ 
 is   $(\eps,\kappa,\x^*)$-SLQC in $x_t$. Moreover, we assume $|f_t(\x)|\leq M,\; \forall t\in[T], \x\in\reals^d~. $
 \end{assumption}
 Note that we assume that $b_0= \text{poly}(1/\eps,\log(T/\delta))$.
 
 \paragraph{Justification of Assumption~\ref{ass:SLQC}} Noisy GLM regression (see Section~\ref{sec:IdealizedPerceptNoisy}),
  is an interesting instance of stochastic optimization problem where Assumption~\ref{ass:SLQC} holds. Indeed according to Lemma~\ref{lem:IdealizedPerceptNoisy}, given $\eps,\delta,T>0$, 
  then  for $b\geq \Omega(\log(T/\delta)/\eps^2)$ samples\footnote{In fact, Lemma~\ref{lem:IdealizedPerceptNoisy}  states that for $b=\Omega(\log(1/\delta)/\eps^2)$, then 
 the error function is SLQC in a \emph{single  decision point}. Using the union bound we can show that for  $b=\Omega(\log(T/\delta)/\eps^2)$ it holds for $T$  decision points},
  the average minibatch function is $(\eps,\kappa,\x^*)$-SLQC in $x_t,\;\forall t\in[T]$,  w.p.$\geq 1-\delta$.

 Local-quasi-convexity of minibatch averages is a plausible assumption when we optimize an expected sum of quasi-convex functions that share  common global minima 
 (or when the different global minima are close by). As seen from the Examples presented in Equation~\eqref{eq:FuncExample}, and in Sections~\ref{sec:IdealizedPercept}, ~\ref{sec:IdealizedPerceptNoisy}, this sum is generally not quasi-convex, but is more often locally-quasi-convex. 
 
 Note that in the general case when the objective is a sum of quasi-convex functions, the number of local minima
 of such objective may  grow \emph{exponentially} with the dimension $d$, see~\cite{auer1996exponentially}. This might imply that a general setup where each $\psi\sim\D$ is quasi-convex may be generally hard. 
%
%
%
 
\subsection{Main Results}
SNGD is presented in Algorithm~\ref{algorithm:SNGD}.   SNGD
is similar to SGD, except we normalize the gradients. The normalization is crucial in order to take advantage of the SLQC assumption, and in order to overcome the hurdles of plateaus and cliffs. 
Following is our main theorem:
\begin{theorem}
\label{theorem:SNGD}
Fix $\delta,\eps,G,M,\kappa>0$. 
Suppose we run  SNGD  with 
$T \geq {\kappa^2\|\x_1 - \x^*\|^2}/{\eps^2}$ iterations, $\eta=\eps/\kappa$, and 
$b \geq\max\{ \frac{M^2\log(4T/\delta)}{2\eps^2}, b_0(\eps,\delta,T)\}$ .
Assume that for $b\geq b_0(\eps,\delta,T)$ then w.p.$\geq1-\delta$ and $\forall t\in[T]$, the function $f_t$ defined
in the algorithm is $M$-bounded, and is also  $(\eps,\kappa,\x^*)$-SLQC in $\x_t$. 
Then, with probability of at least $1-2\delta$,  we have that $f(\bar{\x}_T)-f(\x^*) \le 3 \eps$.
\end{theorem}
We prove of Theorem~\ref{theorem:SNGD} at the end of this section.
\begin{remark}
Since strict-quasi-convexity and $(G,\eps/G,\x^*)$-local-Lipschitzness are equivalent to SLQC (App.~\ref{app:LocallyLip2LocallyQuas}), 
  the theorem implies that $f$ could have \emph{arbitrarily large
    gradients} outside $\B(\x^*,\eps/G)$, yet SNGD is still ensured to
  output an $\eps$-optimal point within ${G^2 \|\x_1-\x^*\|^2}/{\eps^2}$
  iterations. We are not familiar with a similar guarantee for SGD
  even in the convex case.
\end{remark}

\begin{remark}
   Theorem~\ref{theorem:SNGD} requires the minibatch
  size to be $\Omega(1/\eps^2)$. In the context of learning, the
  number of functions, $n$, corresponds to the number of training
  examples. By standard sample complexity bounds, $n$ should also be
  order of $1/\eps^2$.  Therefore, one may wonder, if the
  size of the minibatch should be order of $n$. This is not true, since the
  required training set size is $1/\eps^2$ \emph{times the VC
    dimension of the hypothesis class}. In many practical cases, the
  VC dimension is more significant than $1/\eps^2$, and therefore
  $n$ will be much larger than the required minibatch size. The reason
  our analysis requires a minibatch of size $1/\eps^2$, without
  the VC dimension factor, is because we are just ``validating'' and
  not ``learning''.
\end{remark}

In SGD and for the case of convex functions, even a minibatch of size
$1$ suffices for guaranteed convergence. In contrast, for SNGD we
require a minibatch of size $1/\eps^2$. The theorem below shows
that the requirement for a large minibatch is not an artifact of our
analysis but is truly required.
\begin{theorem}
\label{lem:LowerBound_minibatch}
Let $\eps \in (0,0.1]$; There exists a distribution over convex loss functions, such that running SNGD with minibatch size of $b= \frac{0.2}{\eps}$,  with a high probability, it never reaches an $\eps$-optimal solution 
\end{theorem}
We prove Theorem~\ref{lem:LowerBound_minibatch} in Section~\ref{app:Proof_lem:LowerBound_minibatch}.
The gap between the upper bound of $1/\eps^2$ and the lower bound of
$1/\eps$ remains as an open question.

We now provide a sketch for the proof of Theorem~\ref{theorem:SNGD}:
\begin{proof}[Proof of Theorem~\ref{theorem:SNGD}]
Theorem~\ref{theorem:SNGD} is a consequence of the following two lemmas.
In the first we show that whenever all $f_t$'s are SLQC, there exists some $t$ 
such that  $f_t(\x_t)-f_t(\x^*) \le \eps$.  In the second lemma, we show that for a large enough minibatch size
$b$, then for any $t\in[T]$ we have $f(\x_t) \le f_t(\x_t) + \eps$, and $f(\x^*) \ge f_t(\x^*) - \eps$. Combining these two lemmas we conclude that $f(\bar{\x}_T) - f(\x^*) \leq 3\eps$. 
\begin{lemma}
\label{lem:Ngd4Sngd}
Let $\eps,\delta>0$. Suppose we  run SNGD for $T \ge
 {\kappa^2\|\x_1 - \x^*\|^2}/{\eps^2}$ iterations, $b\geq b_0(\eps,\delta,T)$, and $\eta =
 \eps/\kappa$. Assume that w.p.$\geq 1-\delta$ all $f_t$'s are  $(\eps,\kappa,\x^*)$-SLQC in $x_t$, whenever $b\geq b_0(\eps,\delta,T)$.  Then w.p.$\geq 1-\delta$ we must have some $t \in [T]$ for which $f_t(\x_t)-f_t(\x^*) \le \eps$.
\end{lemma}
Lemma~\ref{lem:Ngd4Sngd} is proved similarly to Theorem~\ref{theorem:NGD}, we defer the proof to Section~\ref{app:Proof_lem:Ngd4Sngd}.

The second Lemma relates $f_t(\x_t)-f_t(\x^*) \le \eps$ to a bound on $f(\x_t)-f(\x^*)$. 
\begin{lemma}
\label{lem:MinibatchSNGD}
Suppose $b\geq \frac{M^2\log(4T/\delta)}{2}\eps^{-2}$ then w.p.$\geq 1-\delta$ and for every $t\in[T]$:
\[
f(\x_t) \le f_t(\x_t) + \eps ~, 
\qquad \text{and also,   }\qquad 
f(\x^*) \ge f_t(\x^*) - \eps ~.
\]
\end{lemma}
Lemma~\ref{lem:MinibatchSNGD} is a direct consequence of Hoeffding's bound (see Section~\ref{app:Proof_lem:MinibatchSNGD}).
Using the definition of $\bar{\x}_T$ (Alg.~\ref{algorithm:SNGD}) , together with Lemma~\ref{lem:MinibatchSNGD} gives:
$$f(\bar{\x}_T)-f(\x^*) \leq  f_t(\x_t)-f_t(\x^*)+2\eps,\; \forall t\in[T] $$
Combining the latter with Lemma~\ref{lem:Ngd4Sngd}, establishes Theorem~\ref{theorem:SNGD}.
\end{proof}
\subsection{Remaining Proofs}
\subsubsection{Proof of Lemma~\ref{lem:Ngd4Sngd}}\label{app:Proof_lem:Ngd4Sngd}
\begin{proof}
First note that if the gradient of $f_t$ vanishes at $\x_t$, then by
the SLQC assumption we must have that $f_t(\x_t)-f_t(\x^*)\le \eps$.
 Assume next that we perform $T$ iterations and the gradient of $f_t$ at $\x_t$
never vanishes in these iterations. 
Consider the update rule of SNGD (Algorithm~\ref{algorithm:SNGD}), then by standard algebra we get:
\[
\|\x_{t+1}-\x^*\|^2 = \|\x_t - \x^*\|^2  - 2 \eta
\inner{\hat{g}_t,\x_t-\x^*} + \eta^2~.
\]
Assume that  $\forall t\in[T]$ we have $f_t(\x_t)-f_t(\x^*) >
\eps$. Take
$\y = \x^* + \left(\eps/\kappa\right) \hat{g}_t$, 
 and observe that $\|\y-\x^*\| \le \eps/\kappa$. 
  Hence the $(\eps,\kappa,\x^*)$-SLQC assumption implies   that $\inner{\hat{g}_t,\y-\x_t} \le 0 $, thus,
\[
\inner{\hat{g}_t,\x^* + \left(\eps/\kappa \right) \hat{g}_t-\x_t} \le 0 
~\Rightarrow~ \inner{\hat{g}_t,\x_t - \x^*} \ge \eps/\kappa~.
\]
This implies that, if we set  $\eta = \eps/\kappa$ then
\begin{align*}
\|\x_{t+1}-\x^*\|^2 &\le \|\x_t - \x^*\|^2  - 2 \eta
\eps/\kappa + \eta^2 \\
&= \|\x_t - \x^*\|^2  - \eps^2/\kappa^2~.
\end{align*}
So, after $T$ iterations for which $f_t(\x_t)-f_t(\x^*) > \eps$ we get 
\[
0 \le \|\x_{T+1}-\x^*\|^2 \le \|\x_1 - \x^*\|^2 - T  \eps^2/\kappa^2 ~,
\]
Therefore, we must have
\[
T \le \frac{\kappa^2 \|\x_1 - \x^*\|^2}{\eps^2} ~.
\]
\end{proof}

\subsubsection{Proof of Lemma~\ref{lem:MinibatchSNGD}}\label{app:Proof_lem:MinibatchSNGD}
\begin{proof}
At each step $t$, the minibatch is being sampled after $\x_t$ and $\x^*$
are fixed. The random variables $f_t(\x_t)$ (resp.  $f_t(\x^*)$) are an average of
$b$ i.i.d. random variables whose expectation is
$f(\x_t)$ (resp.  $f(\x^*)$). These random variables are bounded, since we assume $\forall t,\x,\; |f_t(\x)|\leq M$ (see Thm.~\ref{theorem:SNGD}).
Applying Hoeffding's bound to the $b$ random samples mentioned above,
together with the union bound over $t\in[T]$, and over both sequences
of random variables, the lemma follows.
\end{proof}
\subsubsection{Proof of Theorem~\ref{lem:LowerBound_minibatch}}\label{app:Proof_lem:LowerBound_minibatch}
We will require the following lemma, whose proof is given in App.~\ref{app:Proof_lemma:MarkovChain}.
\begin{lemma}[Absorb probabilities]
 \label{lemma:MarkovChain}
Let $\{X_t\}_{t=1}^\infty$ be  a   Markov chain over states $\{i \}_{i=0}^\infty $, such that $0$ is an absorbing state , and the transition distribution elsewhere is as follows:
\begin{equation}\nonumber
X_{t+1}\vert \{X_t = i\}=
\begin{cases}
	i-1 	&\quad \text{w.p. ~~$p$ } \\ 
	i +1            	&\quad \text{w.p. ~~$1-p$}\\ 
\end{cases}
\end{equation}
Define the absorb probabilities $\alpha_i: = P(\exists t>0: X_t=0\vert X_0=i)$,
then:
$$\alpha_i = \big(\frac{p}{1-p}\big)^i,\qquad \forall i\geq 1$$
\end{lemma}
\begin{proof}
To prove Theorem~\ref{lem:LowerBound_minibatch}, we construct a
counter example in one dimension.  Consider the following distribution $\D$ over loss functions:
\begin{equation}
\label{eq:DitsLosses}
f(\x)=
\begin{cases}
	-0.5\eps \x 	&\quad \text{w.p. ~~$1-\eps$ } \\ 
	\left(1-0.5\eps\right)\max\{\x+3,0\}         	&\quad \text{w.p. ~~$\eps$}\\ 
\end{cases}
\end{equation}
It can be  verified that the optimum of $\E_\D[f(\x)]$  is in $\x^* = -3$ , and that the slope of the expected loss in $(-3,\infty)$ is $0.5\eps$.
Also notice that all points in the segment $[-5,-1]$ are $\eps$-optimal.

Suppose we use SNGD with  a batchsize of $b=\frac{0.2}{\eps}$,
 i.e., we sample the gradient $b$ times at any query point,  and then average the samples and take the sign. Assume that at time $t$ the queried point  is greater than $\x^*=-3$. 
Let $Y_t$ be the averaged gradient over the batch received at time $t$, and   
define $p = P(Y_t\geq 0)$, i.e., the probability that this sign is non-negative. Then the following is a lower  bound on $1-p$:
$$1-p:=P(Y_t<0)\geq (1-\eps)^b = (1-\eps)^{0.2/\eps}~,$$
where $(1-\eps)^b$ is the probability that all $b$ samples are negative. Now, consider the function $G(\eps) =(1-\eps)^{0.2/\eps}$,  
It can be shown that $G$ is monotonically decreasing in $[0,1]$, and that $G(0.1)\geq 0.8$. Therefore, for any $\eps\in [0,0.1]$ we have,
$p\leq 0.2$.

Now, let $\{X_t\}_{t\in [T]}$, be the random variables describing the queries of SNGD under the distribution over loss functions given in Equation~\eqref{eq:DitsLosses}.
Also assume that we start SNGD with $X_1=0$, i.e., at a distance of $D=3$ from the optimum. Then the points that SNGD queries are 
on the  discrete lattice $\{i\eta\}_{i\in \mathcal{Z}}$, and  the following holds:
\begin{equation}\nonumber
X_{t+1}\vert \{X_t = i\eta\}=
\begin{cases}
	(i-1)\eta 	&\quad \text{w.p. ~~$p$ } \\ 
	(i +1)\eta            	&\quad \text{w.p. ~~$1-p$}\\ 
\end{cases}
\end{equation}

Let  $i_0 =  \lceil{-1/\eta}\rceil$, note that $i_0$ is the minimal number of steps required by SNGD to arrive from $X_1=0$, to an $\eps$-optimal solution.
Now in order to analyze the probability that SNGD ever arrives at an $\eps$-optimal point, it is sufficient to consider the Markov chain over the lattice
$\{i\eta\}_{i\in \mathcal{Z}}$ with the state $S_0 = i_0 \eta$, as an absorbing state. 
Using Lemma~\ref{lemma:MarkovChain} we conclude that if we start at $X_1 =0$ then the probability that we ever absorb is:
\begin{align*}
P(\exists t>0: X_t \text{ is $\eps$-optimal }\vert X_0=0) & \leq \left(\frac{p}{1-p} \right)^{\frac{1}{\eta}-1}\\
&\leq \left(\frac{1}{4} \right)^{\frac{1}{\eps}-1}\\
&\leq \left( \frac{1}{4} \right)^{9}~,
\end{align*}
where we used $p\leq 0.2$,  a bound of $G=1$ on the gradients of losses; NGD's learning rate $\eta = \eps/G$,  and $\eps\leq 0.1$.
\end{proof}


\section{Experiments}
\label{section:experiments}
A better understanding of how to train deep neural networks is one
of the greatest challenges in current machine learning and
optimization.  Since learning NN (Neural Network) architectures
essentially requires to solve a hard non-convex program, we have
decided to focus our empirical study on this type of tasks.  As a test
case, we train a Neural Network with a single hidden layer of $100$ units over the
MNIST data set. We use a ReLU activation function, and minimize the
square loss. We employ a regularization over weights with a parameter
of $\lambda = 5\cdot 10^{-4}$.
 
At first we were interested in comparing the performance of SNGD to
MSGD (Minibatch Stochastic Gradient Descent), and to a stochastic
variant of Nesterov's accelerated gradient
method~\cite{sutskever2013importance}, which is considered to be
state-of-the-art.  For MSGD and Nesterov's method we used a step size
rule of the form $\eta_t = \eta_0(1+\gamma t)^{-3/4}$, with $\eta_0 =
0.01$ and $\gamma =10^{-4}$.  For SNGD we used the constant step size
of $0.1$.  In Nesterov's method we used a momentum of $0.95$.  The
comparison appears in Figures~\ref{fig:Comp1},\ref{fig:Comp2}. As expected, MSGD
converges relatively slowly. Conversely, the performance of
SNGD is comparable with Nesterov's method. All methods employed a minibatch size
of 100.
  
Later, we were interested in examining the effect of minibatch size on the performance of SNGD.
We employed SNGD with different minibatch sizes. As seen in Figure~\ref{fig:Minibatch}, the performance improves significantly with the increase of minibatch size.
%

\begin{figure}[t]
\centering
\subfigure[]{ \label{fig:Comp1}
\includegraphics[trim = 17mm 70mm 27mm 60mm, clip,width=0.3\textwidth ]{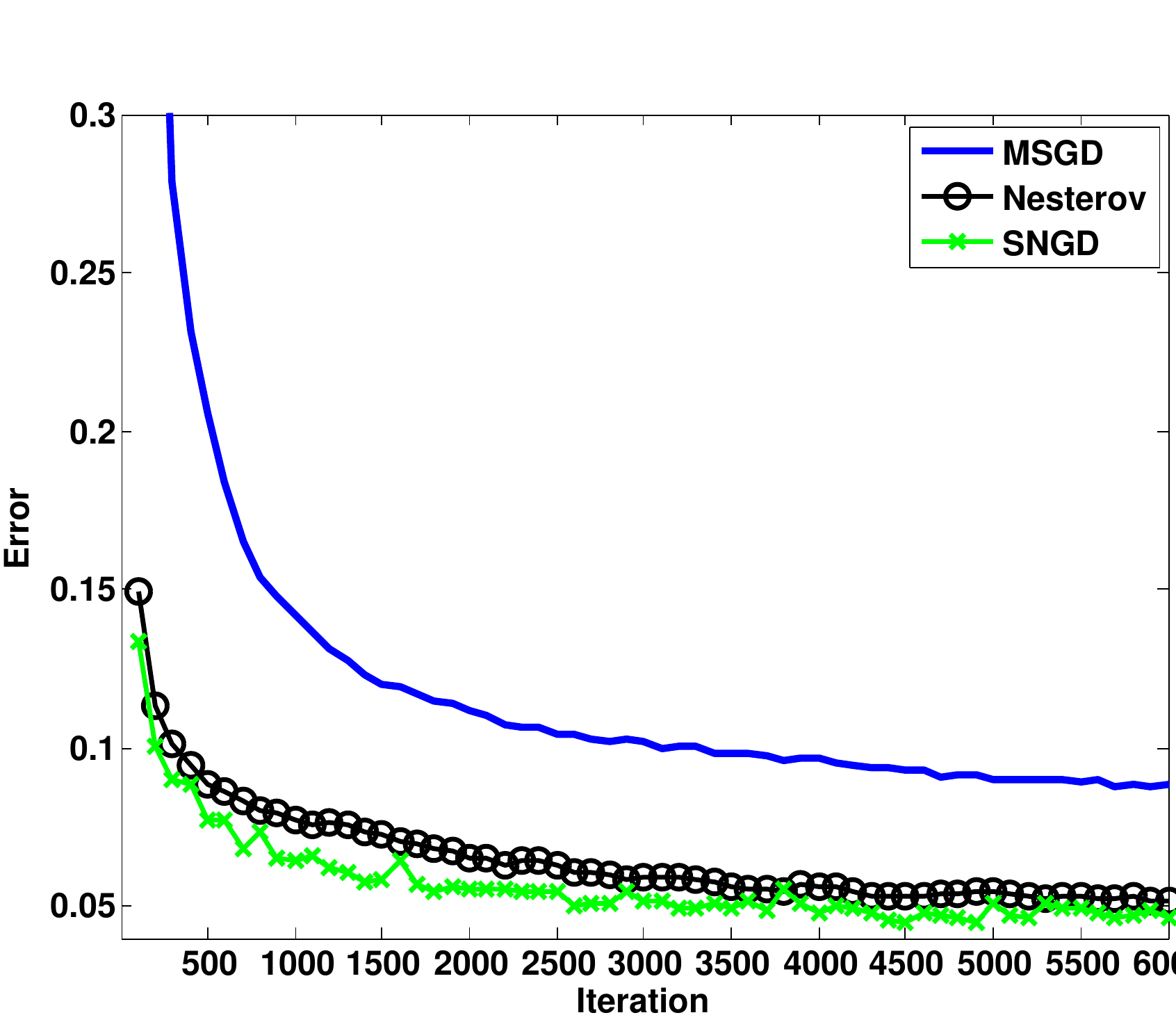}}
\subfigure[]{\label{fig:Comp2}
 \includegraphics[trim = 15mm 70mm 27mm 60mm, clip,width=0.3\textwidth ]{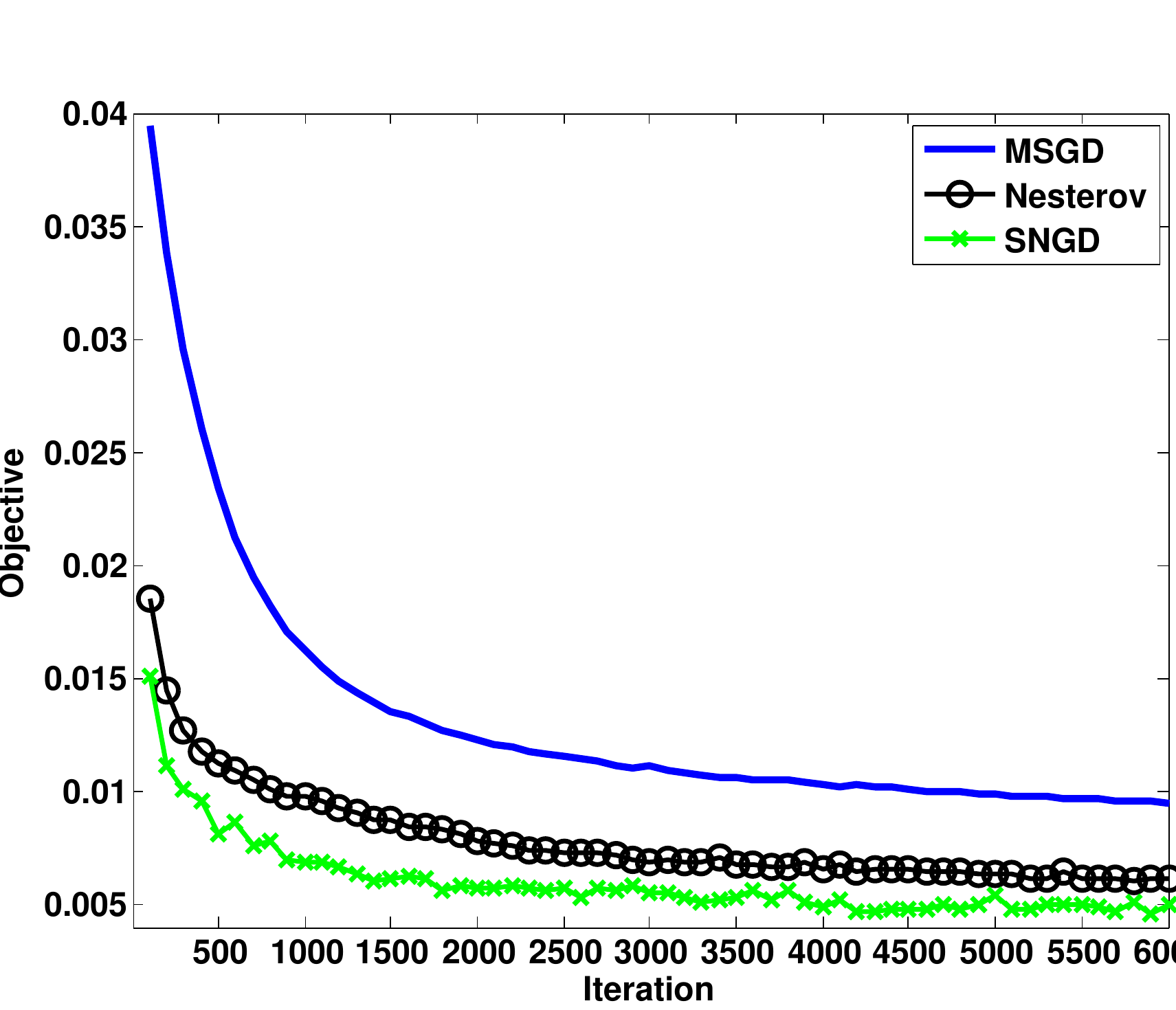}}
 \subfigure[]{ \label{fig:Minibatch}
\includegraphics[trim = 15mm 70mm 27mm 60mm, clip,width=0.3\textwidth ]{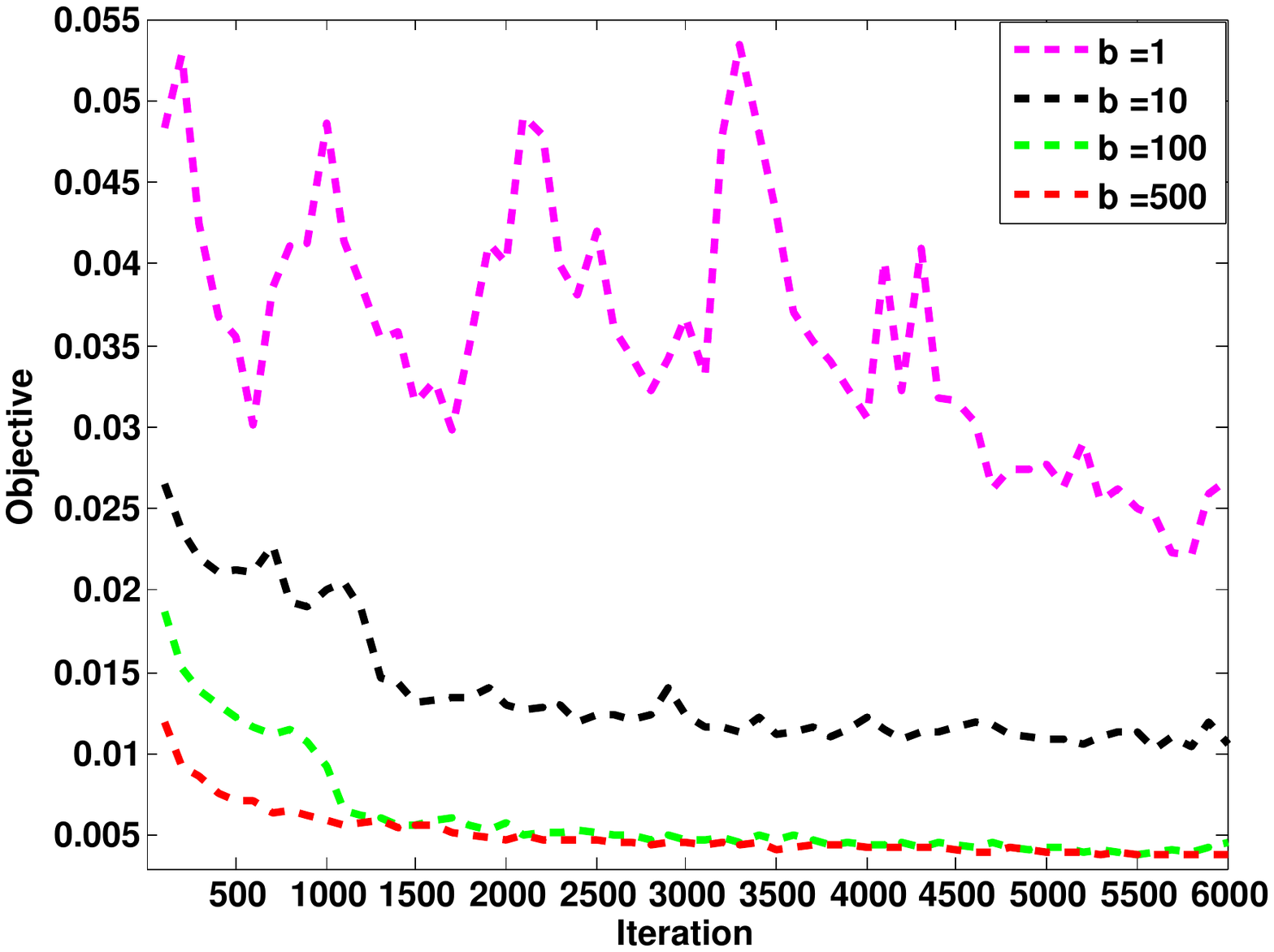}}
\caption{Comparison between optimizations schemes. Left: test
  error. Middle: objective value (on training set). On the Right we compare the objective of SNGD for different minibatch sizes.} 
\end{figure}

\section{Discussion}

We have presented the first provable gradient-based algorithm for stochastic quasi-convex  optimization. This is a first attempt at generalizing the well-developed machinery of stochastic convex optimization to the challenging non-convex problems facing machine learning, and better characterizing the border between NP-hard non-convex optimization and tractable cases such as the ones studied herein. 

Amongst the numerous  challenging questions that remain, we note that there is a gap between the upper and lower bound of the minibatch size sufficient for SNGD to provably converge. 

\bibliographystyle{abbrvnat}
\bibliography{bib}

\newpage
\appendix

\section{Equivalence Between Definitions~\ref{def:qc1} and~\ref{def:qc2}} \label{apdx:defEquiv}
First let us show that~\ref{def:qc1} $\Rightarrow$~\ref{def:qc2}
\begin{proof}[Proof of~\ref{def:qc1} $\Rightarrow$~\ref{def:qc2}]
Let $\x,\y\in \reals^d$ such that $f(\x),f(\y)\leq \lambda$. Let $[\x,\y]$ be the line segment connecting these points; we need to show that $\forall \z\in[\x,\y];\;f(\z)\leq \lambda$. Assume by contradiction that there exists $\z\in[\x,\y]$ such that $f(\z)>\lambda$. Assume w.l.o.g. that $\inner{ \nabla f(\z),\x-\y} \neq 0$ (otherwise we could always find $\z'\in[\x,\y]$ such that  $f(\z')=f(\z)$ and $\inner{ \nabla f(\z'),\x-\y} \neq 0$), and let $\alpha\in(0,1)$ such that $\z=\alpha  \x+(1-\alpha)\y$. By Definition~\ref{def:qc1} the following applies:
\begin{align*}
0&\leq \inner{ \nabla f(\z),\z-\x} = \inner{ \nabla f(\z),\alpha \x+(1-\alpha)\y-\x} = (1-\alpha) \inner{ \nabla f(\z),\y-\x}~,\\
0&\leq \inner{ \nabla f(\z),\z-\y} = \inner{ \nabla f(\z),\alpha \x+(1-\alpha)\y-\y} = -\alpha \inner{ \nabla f(\z),\y-\x}~. 
\end{align*}   
Since $\alpha\in(0,1)$, we conclude that $\inner{ \nabla f(\z),\y-\x}\geq 0$ and also $\inner{ \nabla f(\z),\y-\x}\leq 0$. This is a contradiction since we assumed $\inner{ \nabla f(\z),\y-\x}\neq 0$.
\end{proof}

Let us now show that~\ref{def:qc2} $\Rightarrow$~\ref{def:qc1}
\begin{proof}[Proof of~\ref{def:qc2} $\Rightarrow$~\ref{def:qc1}]
Consider the 1 dimensional function $h(\alpha) = f(\x +
\alpha(\y-\x))$. The derivative of $h$ at $0$ is $h'(0) = \inner{\nabla
  f(\x),\y-\x}$. Therefore, we need to show that if $\y \in S_f(\x)$ then
$h'(0) \le 0$. By the quasi-convex assumption we have that all the
line segment connecting $\x$ to $\y$ is in $S_f(\x)$. Therefore, for every $\alpha \in [0,1]$
we have $h(\alpha) \le h(0)$. This means that 
\[
h'(0) = \lim_{\alpha \to_+ 0} \frac{h(\alpha)-h(0)}{\alpha} \le 0 ~.
\]
\end{proof}  
\section{Local Quasi-convexity of $g$ }\label{app:localQuasi_g}
Here we show that the function $g$ that appears in Equation~\ref{eq:FuncExample} is SLQC. 
Denote $\x^*= (-10,-10)$, let $\eps\in[0,1]$ and let $\x,\y$ such that $g(\x)-g(\x^*)\geq \eps $ and $\|\y-\x^*\|\leq \eps$.
In order to prove SLQC it sufficient to show that $\|g_\x\|>0$, and that $\inner{g_\x,\x-\y}\geq 0$ (we denote $g_\x:=\nabla g(\x)$). 
Deriving $g$ at $x$ we have:
$$g_\x = \nabla g(\x) = (e^{-x_1}/(1+e^{-x_1})^2,e^{-x_2}/(1+e^{-x_2})^2)~.$$
and it is clear that $\|g_\x\|>0,\; \forall \x\in[-10,10]^2$, thus strictness always holds. We divide the proof of $\inner{g_\x,\x-\y}\geq 0$, into cases:
\paragraph{Case 1:} Suppose that $x_1\leq 0,x_2\leq 0$. In this case it is possible to show that the Hessian of $g$ is positive-semi-definite, thus $g$ is \emph{convex} in $[-10,0]^2$. 
Since $g$ is also $1$-Lipschitz, then it  implies that it is $(\eps,1,\x^*)$-SLQC in every $\x\in[-10,0]^2$.
\paragraph{Case 2:} Suppose that  at least one of $x_1,x_2$ is positive, w.l.o.g. assume that $x_1>0$. In this case:
\begin{align*}
\inner{g_\x,\x-\y} 
&\geq
      \sum_{i=1}^2 \frac{e^{-x_i}(x_i+10+(x_i^*-y_i))}{(1+e^{-x_i})^2} \\
&\geq
  \sum_{i=1}^2 \frac{e^{-x_i}(x_i+10-\eps)}{(1+e^{-x_i})^2} \\
&\geq
   \frac{19e^{-10}}{(1+e^{-10})^2}  -\frac{\eps e^{-10+\eps}}{(1+e^{-10+\eps})^2}\\
&\geq
   \frac{e^{-10}}{(1+e^{-10})^2}(19 - \eps e^{\eps}) \\
&>
0~.
\end{align*}
where in the second line we used $\|\y-\x^*\|\leq \eps$. In the third line we used $\eps\in[0,1]$, also $10 = \arg\max_{z\in[0,10]}\frac{e^{-z}(z+9)}{(1+e^{-z})^2}$,
and $\min_{z\in[-10,0]} \frac{e^{-z}(z+10-\eps)}{(1+e^{-z})^2} \geq  -\frac{\eps e^{-10+\eps}}{(1+e^{-10+\eps})^2}$. The fourth line uses $e^{-10+\eps}\geq e^{-10}$, 
and the last line uses $\eps\leq 1$. 

The above two cases establish that $g$ is $(\eps,1,\x^*)$-SLQC in every  $\x\in [-10,10]^2 $,  $\eps\in(0,1]$.

\section{Proof of Lemma~\ref{lem:IdealizedPercept}} \label{app:Proof_lem:IdealizedPercept}
\begin{proof}
Given $\eps\geq 0$, we will show that $\err_m$ is $(\eps,e^W,\w^*)$-SLQC at every $\w\in\B(0,W)$.
Recall  $\phi(z) = (1+e^{-z })^{-1}$, and
consider $\|\w\|\leq W$ such that $\err_m(\w) = \frac{1}{m}\sum_{i=1}^m(y_i - \phi\inner{\w,\x_i})^2\geq \eps$.
Also  let $\v$ be  a point $\eps/e^W$ close to the minima $\w^*$, we therefore have:
\begin{align}\label{eq:app1} 
\inner{\nabla & \err_m(\w),\w-\v} \nonumber \\
&= \frac{2}{m}\sum_{i=1}^m\frac{e^{\inner{\w,\x_i} } } {(1+e^{\inner{\w,\x_i}})^2}( \phi\inner{\w,\x_i}-y_i)(\inner{\w,\x_i}-  \inner{\v,\x_i})\nonumber\\
&= \frac{2}{m}\sum_{i=1}^m\frac{e^{\inner{\w,\x_i} } } {(1+e^{\inner{\w,\x_i}})^2}( \phi\inner{\w,\x_i}- \phi\inner{\w^*,\x_i})(\inner{\w,\x_i}-  \inner{\w^*,\x_i}+\inner{\w^*-\v,\x_i})\nonumber\\
& \geq \frac{2}{m}\sum_{i=1}^m \frac{4e^{\inner{\w,\x_i} } } {(1+e^{\inner{\w,\x_i}})^2}(\phi\inner{\w,\x_i}- \phi\inner{\w^*,\x_i})^2-\frac{\eps e^{-W}}{2}\nonumber\\
& \geq 2e^{-W}\err_m(\w)-\frac{\eps e^{-W}}{2}\nonumber\\
&> 0~.
\end{align}
In the second line we used $y_i = \phi\inner{\w^*,\x_i}$, which holds for the idealized setup. In the third line we used the fact that $\phi(z) $ is monotonically increasing and 
$1/4$-Lipschitz, and therefore $\left(\phi(z)-\phi(z')\right)(z-z')\geq 4\left(\phi(z)-\phi(z')\right)^2$. We also used $|\inner{\w^*-\v,\x_i}|\leq \|\w^*-\v\|\cdot \| \x_i\|\leq \eps e^{-W}$, and 
$|\phi(\w,\x_i)- \phi(\w^*,\x_i)|\leq 1$; Finally we used $\max_z \frac{e^{z } } {(1+e^{z})^2}\leq 1/4$. The fourth line uses $\min_{\|z\|\leq W}\frac{4e^z}{(1+e^z)^2}\geq e^{-W}$. 
The last line follows since we assume $\err_m(\w)\geq \eps$. 
The strictness is immediate since $\inner{\nabla \err_m(\w),\w-\v}>0$, therefore, the above establishes SLQC.
\paragraph{We will now show that $\err_m$ is generally not quasi-convex:} Consider the idealized setup with two samples $(\x_1,y_1),(\x_2,y_2)$ where $\x_1 = (0,-\log4),\x_2=(-\log 4,0)$ and 
$y_1=y_2=1/5$, The error function is therefore:
$$\err_m(\w) = \frac{1}{2}\left(\frac{1}{5}- \frac{1}{1+e^{-\inner{\w,\x_1}}}\right)^2+\frac{1}{2}\left(\frac{1}{5}- \frac{1}{1+e^{-\inner{\w,\x_2}}}\right)^2~.$$
and it can be verified that the optimal predictor is $\w^* = (1,1)$, yielding $\err_m(\w^*)=0$.
Now let $\w_1 = (3,1), \w_2=(1,3)$, it can be shown that $\err_m(\w_1) = \err_m(\w_2) \leq 0.018 $, yet $\err_m(\w_1/2+\w_2/2) \geq 0.019 $. Thus $\err_m$ is not quasi-convex.
\end{proof}

\section{Proof of Lemma~\ref{lem:IdealizedPerceptNoisy}} \label{app:Proof_lem:IdealizedPerceptNoisy}
\begin{proof}
Since we are in the noisy idealized setup, and $\forall i,\;y_i\in[0,1]$  the folllowing holds 
$$y_i = \phi\inner{\w^*,\x_i} + \xi_i~.$$
where  $\{\xi_i\}_{i=1}^m$ are zero mean, independent and bound random variables,  $\forall i\in[m],\;|\xi_i|\leq 1$. Therefore $\err_m$ can be written as follows:
\begin{align*}
\err_m(\w)&= \frac{1}{m}\sum_{i=1}^m(y_i - \phi\inner{\w,\x_i})^2\\
&=  \frac{1}{m}\sum_{i=1}^m(\phi\inner{\w^*,\x_i} - \phi\inner{\w,\x_i})^2 + \frac{1}{m}\sum_{i=1}^m \xi_i\theta_i(\w)+c~.
\end{align*}
where $c=\frac{1}{m}\sum_{i=1}^m \xi_i^2$, and $\theta_i(\w) =2\left(  \phi\inner{\w,\x_i}-\phi\inner{\w^*,\x_i}\right)$. We therefore have:
\begin{align*}
\err_m(\w)-\err_m(\w^*)& =  \frac{1}{m}\sum_{i=1}^m(\phi\inner{\w^*,\x_i} - \phi\inner{\w,\x_i})^2 + \frac{1}{m}\sum_{i=1}^m \xi_i\theta_i(\w)~.
\end{align*}
Now fix $\eps>0$, and let $\w$ be a fixed point in $\B(0,W)$ such that $\err_m(\w)-\err_m(\w^*) \geq  \eps$. Also  let $\v$ be  a point $\eps/ e^{W}$ close to the minima $\w^*$, we therefore have:
\begin{align}\label{eq:app2}
\inner{\nabla &\err_m(\w),\w-\v}  \nonumber\\
&= \frac{2}{m}\sum_{i=1}^m\frac{e^{\inner{\w,\x_i} } } {(1+e^{\inner{\w,\x_i}})^2}( \phi\inner{\w,\x_i}-y_i)(\inner{\w,\x_i}-  \inner{\v,\x_i})\nonumber\\
&= \frac{2}{m}\sum_{i=1}^m\frac{e^{\inner{\w,\x_i} } } {(1+e^{\inner{\w,\x_i}})^2}( \phi\inner{\w,\x_i}- \phi\inner{\w^*,\x_i}-\xi_i )(\inner{\w,\x_i}-  \inner{\w^*,\x_i}+\inner{\w^*-\v,\x_i})\nonumber\\
& \geq \frac{2}{m}\sum_{i=1}^m \frac{4e^{\inner{\w,\x_i} } } {(1+e^{\inner{\w,\x_i}})^2}(\phi\inner{\w,\x_i}- \phi\inner{\w^*,\x_i})^2-\frac{\eps e^{-W}}{2}+ \frac{1}{m}\sum_{i=1}^m \xi_i \lambda_i(\w)\nonumber\\
& \geq \frac{2}{m}\sum_{i=1}^m \frac{4e^{\inner{\w,\x_i} } } {(1+e^{\inner{\w,\x_i}})^2}\big[(\phi\inner{\w,\x_i}- \phi\inner{\w^*,\x_i})^2+\xi_i\theta_i(\w)\big]-\frac{\eps e^{-W}}{2}
+ \frac{1}{m}\sum_{i=1}^m \xi_i \tilde{\lambda}_i(\w)\nonumber\\
& \geq 2e^{-W}\left( \err_m(\w)-\err_m(\w^*)\right)-\frac{\eps e^{-W}}{2}+ \frac{1}{m}\sum_{i=1}^m \xi_i \tilde{\lambda}_i(\w)\nonumber\\
&\geq \frac{3}{2}\eps e^{-W} + \frac{1}{m}\sum_{i=1}^m \xi_i \tilde{\lambda}_i(\w)~.
\end{align}
where we denote $\lambda_i(\w) = \frac{8e^{\inner{\w,\x_i} } } {(1+e^{\inner{\w,\x_i}})^2}( \inner{\w^*,\x_i}-\inner{\w,\x_i})$, and $\tilde{\lambda_i}(\w) = \lambda_i(\w) - \frac{8e^{\inner{\w,\x_i} } } {(1+e^{\inner{\w,\x_i}})^2}\theta_i(\w) $. The argumentation justifying the above inequalities is the same as is 
done for Equation~\eqref{eq:app1} (see Appendix~\ref{app:Proof_lem:IdealizedPercept}). 
According to Equation~\eqref{eq:app2},  the lemma is established if we can show that 
$$\frac{1}{m}\sum_{i=1}^m \xi_i \tilde{\lambda}_i(\w) \geq -\eps e^{-W}$$
The  $\{\xi_i\}_{i=1}^m$ are zero mean and independent, and $ |\xi_i \tilde{\lambda}_i(\w)|\leq 4(W+1)$, thus applying Heoffding's bound we get that the above does hold for  
$m\geq \frac{8e^{2W}(W+1)^2}{\eps^2} \log(1/\delta)$. Note that in bounding $|\xi_i \tilde{\lambda}_i(\w)|$, we used $|\xi_i|\leq 1$, also  $\w,\w^*\in \B(0,W)$, and  $\max_z \frac{e^z}{(1+e^z)^2}\leq \frac{1}{4}$.
\end{proof}

\section{Locally-Lipschitz and Strictly Quasi-Convex are SLQC  }~\label{app:LocallyLip2LocallyQuas}
In order to show that strictly quasi-convex function which is also $(G,\eps/G,\x^*)$-Lipschitz, is SLQC, we require the following lemma:
\begin{lemma}\label{lem:qc1OLD}
Let $\z \in\reals^d$, and assume that $f$ is $(G,\eps/G,\z)$-Locally-Lipschitz. Then, for every $\x$  with $f(\x) - f(\z) > \eps$ we have
  $B(\z,\eps/G) \subseteq S_f(\x)$
\end{lemma}
\begin{proof}
Recall the  notation $S_f(\x)=\{\y:f(\y)\leq f(\x)\}$. By Lipschitzness, for every $\y \in B(\z,\eps/G)$ we have
$f(\y) \le f(\z) + \eps$. Combining with the assumption that $f(\z)+\eps < f(\x)$
we obtain that $\y \in S_f(\x)$. 
\end{proof}
Therefore, if $f(\x)-f(\x^*)\geq \eps$, then $\forall y \in \B(\x^*,\eps/G)$ it holds that $f(\x)-f(\y)\geq 0$, 
and since $f$ is strictly quasi-convex, the latter means that $\inner{\nabla f(\x),\y-\x}\leq 0$, and $\|\nabla f(\x)\|>0$. Thus $(\eps,G,\x^*)$-SLQC is established.

\section{Proof of Theorem~\ref{theorem:NGDSmooth}}
\label{app:proofNGDsmooth}
The key lemma, that enables us to attain faster rates for smooth functions is the following:
\begin{lemma}\label{lem:qc1Smooth}
Let $\x^*$ be a global minima of $f$.
Also assume that $f$ is   $(\beta,\sqrt{2\eps/\beta},\x^*)$-locally-smooth. Then, for every $\x$  with $f(\x) - f(\x^*) > \eps$ we have
  $B(\x^*,\sqrt{2\eps/\beta}) \subseteq S_f(\x)$.
\end{lemma}
\begin{proof}
Combining the definition of local-smoothness (Def.~\ref{def:localSmoothness}) together with $\nabla f(\x^*)=0$ we get
$$|f(\y)-f(\x^*)|\leq \frac{\beta}{2}\|\y-\x^*\|^2,\quad \forall \y\in \B(\x^*,\sqrt{2\eps/\beta})$$
Therefore, for every $\y \in B(\x^*,\sqrt{2\eps/\beta})$ we have
$f(\y) \le f(\x^*) + \eps$. Combining with the assumption that $f(\x^*)+\eps < f(\x)$
we obtain that $\y \in S_f(\x)$. 
\end{proof}
The proof of  Theorem~\ref{theorem:NGDSmooth} follows the same lines as the proof of Theorem~\ref{theorem:NGD}.
The main difference is that whenever $f(\x_t) - f(\x^*)\geq \eps$, we use Lemma~\ref{lem:qc1Smooth} and quasi-convexity to show that for $\y = \x^* + \sqrt{2\eps/\beta} \hat{g}_t$ it follows that
$$\inner{\nabla f(\x_t), \y - \x_t}\leq 0~. $$ We  therefore omit the details of the proof.


\section{Proof of Lemma~\ref{lemma:MarkovChain}}\label{app:Proof_lemma:MarkovChain}
\begin{proof}
Using the stationarity and Markov property of the chain, we can write the following recursive equations for the absorb probabilities:
\begin{align}
&\alpha_i = (1-p)\alpha_{i+1}+p \alpha_{i-1}, \qquad \forall i>1 \label{equation:ChainRecursive}\\
&\alpha_{1} = (1-p)\alpha_{2}+p \label{equation:ChainBase}
\end{align}
Lets guess a solution of the form, $\alpha_i =c_0 \rho^i$, where $\rho$ is the decay parameter of the absorb probabilities. By inserting this solution into equation~\eqref{equation:ChainRecursive} we get an equation for $\rho$:
$$(1-p)\rho^2-\rho+p=0~.$$
And it can be validated that the only nontrivial solution is $\rho = \frac{p}{1-p}$, using the latter $\rho$ in equation~\eqref{equation:ChainBase} we get $c_0=1$, and therefore we conclude that:
$$\alpha_i = \big(\frac{p}{1-p}\big)^i,\qquad \forall i\geq 1$$
\end{proof}
\section{A Broader Notion of Local-Quasi-Convexity}~\label{app:def:LocalQC}
Definition~\ref{def:LocalQC} describes a rich family of function, as depicted in Section~\ref{sec:IdealizedPercept}, and~\ref{sec:IdealizedPerceptNoisy}. However, it is clear that it does not capture piecewise constant and quasi-convex functions, such as the zero-one loss, or the Perceptron problem.

In some cases, e.g. the Perceptron problem, we may have an access to a \emph{direction oracle}, $\G:\reals^d\mapsto\reals^d$. This oracle is a proxy to the gradient, aiming us in a global ascent (descent) direction.
Following is a broader definition of locally-quasi-convex functions:
\begin{definition}\textbf{(Local-Quasi-Convexity)}\label{def:LocalQC_broad}
Let $\x,\z\in \reals^d$,  $\kappa,\eps>0$. Also let $\G:\reals^d\mapsto\reals^d$.
 We say that  $f: \reals^d \mapsto \reals$ is  $(\eps,\kappa,\z)$-Strictly-Locally-Quasi-Convex (SLQC) in $\x$, with respect
 to the direction oracle $\G$, if at least one of the following applies:
\begin{enumerate}
\item $f(\x)- f(\z)\leq  \eps~.$
\item $\|\G(\x)\|> 0$, and for every $\y\in\B(\z,\eps/\kappa)$ it holds that $\inner{\G(\x), \y-\x}\leq 0~.$
\end{enumerate}
\end{definition}
Thus, Definition~\ref{def:LocalQC}, is a private case of the above, which takes the gradient of $f$ to be the direction oracle. 
Note that we can show that NGD/SNGD  and their guarantees still hold for SLQC functions with a direction oracle.
The algorithms and proofs are very similar to the ones that appear in the paper, and we therefore omit the details.

In the following section we illustrate a scenario that fits the above definition.
\subsection{The $\gamma$-margin Perceptron}\label{sec:IdealizedPerceptBroad}
In this setup we have a collection of $m$ samples $\{(\x_i,y_i) \}_{i=1}^m \in \B_d\times \{0,1\}$
,and we are guaranteed to have  $\w^* \in \reals^d$ such that: $ y_i \inner{\w^*,\x_i}\geq \gamma,\; \forall i\in[m]$.
Thus, using the sign of  $\inner{\w^*,x_i}$ as a predictor, it classifies all the points correctly (with a margin of $\gamma$).

Letting $\phi$ be the zero-one loss $\phi(z) = \ind{z\geq0}$, we measure 
the performance of a predictor $\w\in\reals^d$, by the average (square) error over all samples,
\begin{align}\label{eq:ErrPerceptronBroad}
\err_m(\w) = \frac{1}{m}\sum_{i=1}^m \left(y_i-\phi\inner{\w,\x_i}\right)^2~. 
\end{align}
Clearly, the gradients of $\err_m(\w)$ vanish almost everywhere. Luckily, from the convergence analysis 
of the Perceptron (see e.g.~\cite{Kalai}), we know the following to be a  direction oracle for  $\err_m(\w)$:
\begin{align}\label{eq:DirectiveOracle}
\G(\w) = \frac{1}{m}\sum_{i=1}^m \left(\phi\inner{\w,\x_i}-y_i\right)\x_i~. 
\end{align}
  
The next lemma states that in the above setup, the error function is SLQC with respect to $\G$. This implies that Algorithm~\ref{algorithm:NGD} finds an $\eps$-optimal minima of $\err_m(\w)$ within $\poly(1/\eps)$ iterations. 
\begin{lemma}\label{lem:IdealizedPerceptBroad}
Consider the $\gamma$-margin Perceptron problem. Then the error function appearing in Equation~\eqref{eq:ErrPerceptronBroad} is $(\eps,2/\gamma,\w^* )$-SLQC in $\w$, $\forall \eps\in(0,1),\;\forall \w \in\reals^d$ 
with respect to the direction oracle appearing in Equation~\eqref{eq:DirectiveOracle}.
\end{lemma}
\begin{proof}
Given $\eps\in (0,1)$, we will show that $\err_m$ is $(\eps,2/\gamma,\w^*)$-SLQC at every $\w\in\reals^d$.
Consider $\w\in\reals^d$ such that $\err_m(\w) = \frac{1}{m}\sum_{i=1}^m(y_i - \phi\inner{\w,\x_i})^2\geq \eps$.
Also  let $\v$ be  a point $\gamma \eps/2$ close to the minima $\w^*$, we therefore have:
\begin{align}
\inner{\G(\w)&,\w-\v} \nonumber \\
&= \frac{1}{m}\sum_{i=1}^m( \phi\inner{\w,\x_i}-y_i)(\inner{\w,\x_i}-  \inner{\v,\x_i})\nonumber\\
&= \frac{1}{m}\sum_{i=1}^m( \phi\inner{\w,\x_i}- \phi\inner{\w^*,\x_i})(\inner{\w,\x_i}-  \inner{\w^*,\x_i}+\inner{\w^*-\v,\x_i})\nonumber\\
& \geq \frac{\gamma}{m}\sum_{i=1}^m (\phi\inner{\w,\x_i}- \phi\inner{\w^*,\x_i})^2-\gamma\eps/2 \nonumber\\
&\geq \gamma\eps-\gamma\eps/2 \nonumber\\
&> 0~.
\end{align}
In the second line we used $y_i = \phi\inner{\w^*,\x_i}$, which holds by our assumption on $\w^*$. In the fourth line we used the fact that $( \phi\inner{\w,\x_i}- \phi\inner{\w^*,\x_i})(\inner{\w,\x_i}-  \inner{\w^*,\x_i})\geq \gamma ( \phi\inner{\w,\x_i}- \phi\inner{\w^*,\x_i})^2$, which holds since $\w^*$ is a minimizer with a $\gamma$-margin.
We also used $\err_m(\w)\leq 1$, and $|\inner{\w^*-\v,\x_i}|\leq \|\w^*-\v\|\cdot \| \x_i\|\leq \gamma \eps/2$.
Lastly, we use $\err_m(\w)\geq \eps$.

The strictness is immediate since $\inner{\G(\w),\w-\v}>0$, therefore, the above establishes SLQC.
\end{proof}

\end{document}